%% file: main_arxiv.tex
\documentclass[letterpaper,twocolumn,10pt]{article}
\usepackage{usenix-2020-09}

\usepackage{balance}
\usepackage{graphicx} 
\usepackage{xcolor}
\usepackage[numbers, sort&compress]{natbib}
\usepackage{amsmath,amsthm,amsfonts}
\usepackage{soul}
\usepackage{enumitem}
\usepackage{thmtools, thm-restate}
\usepackage{color-edits}
\addauthor{conditional}{blue}  

\makeatletter
\renewcommand\paragraph{\@startsection{paragraph}{4}{\z@}%
  {1.5ex \@plus1ex \@minus.2ex}
  {-0.4em}
  {\normalfont\normalsize\bfseries}}
\makeatother

\begin{document}
\date{}

\title{Quota Marketplace: Dynamic Pricing for Efficient Allocation of ML Training Resources}
  
\author{
  {\rm Balasubramanian Sivan}$^{*,1}$ \ 
  {\rm Renato Paes Leme}$^{*,1}$ \ 
  {\rm Mihai Tiuca}$^{*,1}$ \ 
  {\rm Ian McFarlane}$^{1}$ \ 
  {\rm Vasilis Gkatzelis}$^{1,2}$\\
  {\rm Nehal Mehta}$^{1}$ \ 
  {\rm Soheil Hassas Yeganeh}$^{1}$ \ 
  {\rm Vahab Mirrokni}$^{1}$ \ 
  {\rm Amin Vahdat}$^{1}$ \ 
  ({\small $^1$Google},
  {\small $^2$Drexel University}) 
}

\maketitle

\renewcommand{\thefootnote}{\fnsymbol{footnote}}
\footnotetext[1]{Equal contribution.}
\renewcommand{\thefootnote}{\arabic{footnote}}

\begin{abstract}
\input{abstract}
\end{abstract}

\input{intro}
\input{system}

\input{empirical}

\input{challenges}

\input{theory}
\input{future_research}
\input{related}
\balance
\input{acknowledgements}
\clearpage
\bibliographystyle{plainnat}
\bibliography{ref}

\newpage
\appendix

\end{document}

%% file: abstract.tex
The escalating demand for Machine Learning  (ML) training resources in recent years has resulted in a substantial gap between the high demand and the available supply. Efficient allocation of these scarce and expensive resources is crucial for organizations to maximize their return on investment. Existing resource allocation mechanisms, like Karma~\cite{VFACKT23}, are designed to guarantee Pareto efficiency and max-min fairness in settings with dynamic (time-varying) user demands, but fail to preserve these key properties in the presence of demands with heterogeneous values. Given the ubiquity and inevitability of heterogeneity in organizational values of different workloads, effective resource allocation policies must accommodate these variations.

In this paper, we describe the design, implementation, deployment, and theoretical analysis of Quota Marketplace, a market-based mechanism to efficiently allocate ML training chips (like GPUs), explicitly addressing scenarios with demands of heterogeneous value. We detail the implementation of this mechanism within Google and present metrics that demonstrate its impact. We also discuss many business-critical requirements that the Quota Marketplace handles quite effectively, and document the gains and opportunities it has unlocked. We establish theoretically how this market-based approach achieves the essential properties of Pareto efficiency and max-min fairness by allowing the users to express the value of their workloads and enabling dynamic resource pricing based on supply and demand fluctuations. Ultimately, the market facilitates resource allocation that aligns with organizational priorities.

%% file: intro.tex
\section{Introduction}\label{sec:intro}

The demand for Artificial Intelligence (AI) compute has experienced unprecedented growth, increasing by over tenfold annually for eight consecutive years, resulting in a 100 million-fold increase over that 8-year period~\cite{CloudNext25}. Major technology corporations have publicly disclosed substantial capital expenditure (CapEx) budgets, tens of billions of dollars each for 2025, to address this growing demand. This exponential growth in demand creates a widening gap between the demand and supply of AI compute resources, specifically Machine Learning (ML) accelerators like GPUs. Consequently, due to the high cost and limited availability of these resources, organizations must implement return-on-investment (ROI) maximizing, efficient allocation strategies for ML accelerators among individual teams and products. In this paper, we describe Quota Marketplace (QM) -- a market-based system designed, implemented, theoretically analyzed, and fully deployed at Google, to allocate ML chips to teams across business units. The deployed system allocates several hundreds of thousands of ML accelerators for \emph{ML training workloads} across all the business units in Google. The footprint of the Quota Marketplace is a double-digit percentage of the \emph{entire ML fleet of Google}.

We begin with an overview of QM's position in the ML scheduling stack and the notion of pools. A detailed description of QM concepts and implementation follow in Section~\ref{sec:system}.

\paragraph{QM and the ML Training Scheduling Stack in a Nutshell.} ML compute resources (ML accelerators or ML chips), much like non-ML compute resources (CPUs), have traditionally been partitioned into \emph{static} pools. Each pool corresponds to a business unit (BU) with distinct goals, cost accounting, and resources, and the pool administrators/planners explicitly assign accelerators to teams within that BU. The demand assessment and prioritization procedure typically takes place quarterly or semi-annually and it determines the allocation of chips across pools and across teams within a given pool. This assignment then remains more or less static (especially the distribution across pools) until the next assessment. 

\begin{figure}[h!]
    \centering
    \includegraphics[width=1\linewidth]{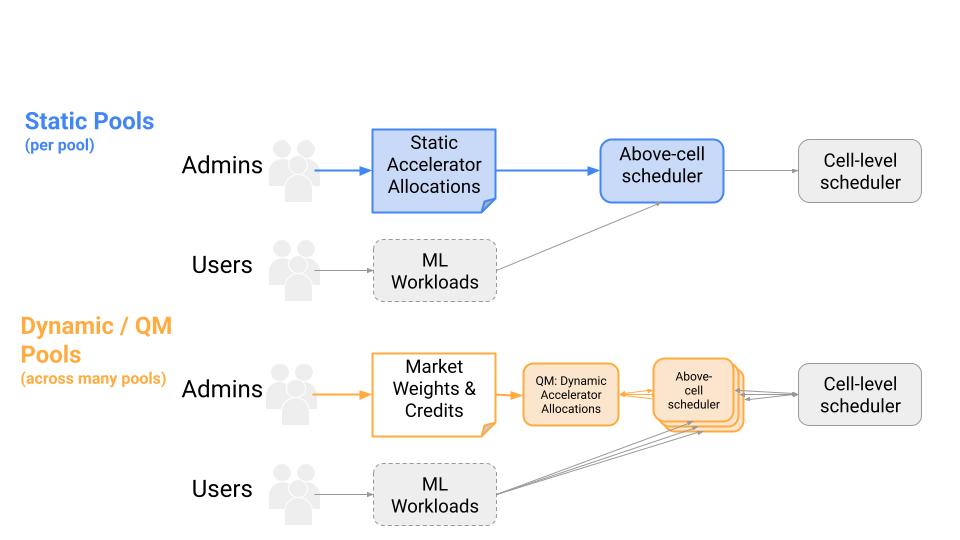}
    \caption{QM in the ML (Training) Scheduling Stack}
    \label{fig:introQM}
\end{figure}

Given the rising prominence of ML training workloads and their need for batch compute (which tolerates preemptions much better than serving workloads)—the Quota Marketplace was conceptualized to make the accelerator allocations both dynamic and agile, and economically more efficient, by incorporating a cross-pool view into allocation decisions. Figure~\ref{fig:introQM} is a short summary of QM's position in the ML scheduling stack. Conceptually, when compared to static pools, there are two main differences  in a \emph{dynamic pool} (a) QM replaces static accelerator allocations with dynamic ones, and (b) QM makes these allocation decisions across all the dynamic pools instead of the per-pool allocation decisions in the static pools. 

In dynamic pools, which is where QM operates, company-level administrators are tasked with allocating market weights to the different pools that reflect the business impact delivered by these units. Per-pool administrators are tasked with allocating credit, which are location agnostic, resource-type agnostic tokens, to the various teams within their business unit, which again reflect the impact the various teams deliver. The toilsome, communication-heavy, quarter-long human-driven negotiations involved in the traditional prioritization exercises of static allocation are now replaced by supply-demand-based resource allocation decisions both across teams in a pool, and also across pools, at a frequency of roughly every 1 minute. Beyond the several orders of magnitude more-responsive resource allocation that QM brings, QM's cross-pool view while clearing supply and demand automatically makes the allocation economically more efficient. 

In a little more detail, QM determines resource allocation by computing clearing prices for each resource type at any given instant. Prices are dynamically chosen to equate \emph{real-time} demand from scheduler queues with the \emph{real-time} supply of the relevant ML accelerator. The market ``clears'' roughly every minute (often faster than that). By comparing real-time team bids generated by an automated bidding algorithm (that QM provides teams with) against these prices, QM decides which teams receive resources and in what quantity. Consequently, quota in pools (aka dynamic pools) is not guaranteed, but is derived from market conditions, making it well-suited for exploratory workloads and tasks with flexible SLAs, as is the case with most ML training workloads.

We now discuss the many benefits and opportunities unlocked by the QuotaMarketplace.

\paragraph{Agility in the Face of Highly Fluctuating Supply and Demand.}
The available supply of ML chips is highly dynamic. New hardware capacity comes online daily, often on an imperfect schedule. Resources are also frequently and unpredictably clawed back, at short notice, for a variety of reasons, including, for example, being able to serve production load due to a sudden surge in traffic. On the other hand, human-driven demand assessment and prioritization procedures, as mentioned earlier, are typically conducted quarterly or semi-annually. To maintain allocation speed, various processes used to rely on readily available resource buffers that remain significantly underutilized and unprioritized. A robust system must respond to both rapid capacity addition and swift clawback/reallocation in the order of minutes, rather than quarters. Doing so would directly contribute to significant reduction in the buffer maintained by production teams, as they could promptly claw back resources upon detecting impending traffic surges. As we make the point with data in Section~\ref{sec:empirical}, this reclaimed ``bonus'' capacity is substantial, sometimes on the order of the total committed capacity, and its underprioritized usage results in significant opportunity cost. In other words, QM results in a big increase in company-prioritized occupancy of chips, rather than opportunistic occupancy of chips (which can be thought of as randomly assigning chips), which we demonstrate with data in Section~\ref{sec:empirical}. That said, QM also increases the overall occupancy of chips (not just company-prioritized occupancy), because it avoids the siloing inherent to static pools (also demonstrated with data in Section~\ref{sec:empirical}). Siloing results in unused capacity, which is often not easily available to use even opportunistically for other pool's workloads.

The demand for ML chips is no less dynamic than their supply. Indeed, deriving the specific ML chip demand for each team directly from the company's business priorities is a challenging task. Our fluctuating usage patterns range from scheduled model training runs to short-duration bursts for demos, exploratory research, and conference submission sprints, exacerbating this problem. Therefore, individuals responsible for resource allocation at a company level (i.e., across all the pools), while knowledgeable about the high-level business goals, often lack granular insight into the time sensitivity, spatial flexibility (i.e., where chips are needed), usage volume, and relative priority of individual requests. Conversely, engineering teams understand the detailed requirements of each workload but lack the broader perspective needed to prioritize across different projects and business units. Therefore, the inherent assumption behind the static pools' prioritization procedures, that we are able to get reasonably accurate predictions of chip requirements for an entire quarter, is highly questionable. Further, regardless of how good our predictions are, a static allocation of chips is very suboptimal in the presence of highly dynamic demand patterns. What's needed is a mechanism to swiftly move allocations in response to real-time demands. QM accomplishes this by enabling continuous, dynamic bidding based on real-time needs.

\paragraph{Enhanced Productivity and Engineering Resource Conservation.} QM significantly reduces the frequency of teams having to negotiate with their BU's pool administrators for more capacity in the middle of a quarter or a half year. In static pools, when there is a demand surge (which, as discussed, is not uncommon), it inevitably leads engineers to interact with the administrators to temporarily get more capacity. It is infeasible to solve this problem by statically allocating chips that meets everyone's \emph{peak} demands. Dynamic allocation is inevitable. In QM, these negotiations are relatively rare, and even when it is necessary, QM makes the administrator's tasks much easier --- simply increase the credit allocation of the concerned team (if their ask is deemed deserving), without concerns for which team to reclaim chips from, to fund this request. Unlike chips and chip hours, credits are (a) not zero-sum and (b) \{resource-type, location\}-agnostic. Therefore increasing one team's credits without altering any other team's credits is perfectly valid.

\paragraph{Incentive Alignment in the Face of Intense Scarcity.}
ML compute resources are very scarce (demand far exceeds supply), and scarcity of resources inevitably drives participants toward strategic behavior. Consequently, a robust resource allocation mechanism must align user incentives with the designer's objectives, such as efficiency and fairness.

To achieve these goals, resource allocation systems should incentivize demand deferral to less congested time periods, geographic locales, or resource-types. A good mechanism should preempt ``use-it-or-lose-it'' scenarios by incentivizing teams to relinquish resources that they do not have a high value for so that they can be consumed by teams with more pressing needs instead. Similarly, given the variety of available resource types that arises from different accelerator generations, the mechanism should incentivize teams to use less congested types of resources.

Finally, the system must provide clear signals that enable users to plan effectively where and when to run their workloads. In QM, prices serve a dual purpose: they provide (a) economic incentives, and (b) a transparent signal of resource congestion and how it varies across time and location.

\paragraph{Guaranteed Efficiency and Fairness.}
When allocating scarce resources, a central systems goal is to achieve efficiency and fairness properties. A fundamental notion of efficiency is Pareto efficiency: an allocation is Pareto efficient if there is no other allocation that would result in (a) weakly higher utility for all users and (b) strictly higher utility for at least one user. A less stringent requirement, directly implied by Pareto efficiency, is to avoid situations where resources are unused despite existing demand. In terms of fairness, the standard goal when allocating computational resources is to achieve (approximate) max-min fairness, which aims to maximize the minimum amount of resources (or utility) any user is allocated. In a large organizational context, this is commonly extended to a weighted max-min notion that normalizes the ``fair share'' for each team based on its potential to positively impact core business priorities. 
In Section~\ref{sec:theory} we define Pareto efficiency and max-min fairness precisely, and show how QM guarantees both. 

\paragraph{Shock Absorber.} QM functions as a shock absorber for the entire corporation, capable of handling significant capacity fluctuations arising from external events. Whether due to the launch of a marquee LLM application that has gained significant traction from customers, requiring hundreds of thousands of accelerators, a high-priority external demo, or surges in cloud customer demand, the system effectively absorbs these shocks without manual intervention. This is not just the earlier point about the agility in handling these requests, but also about distributing the benefits or pains from additions or clawbacks in the fairest way possible according to organizational priority. Shocks can go the other direction as well: QM enables the economically efficient usage of highly fragmented resources added to the system for temporary periods of time, either from serving buffer, or from holding pools.

\paragraph{ Prior Approaches.} Two broad classes of approaches in prior work are ``static allocation'' and ``chip-hour'' mechanisms. 

\emph{Static allocation.} 
This is the idea behind the traditional static pools discussed earlier. Consider a simple setting with a single type of resource, like an ML accelerator, which has to be divided among a number of users. One of the most popular solutions is to statically allocate~\cite{atikoglu2012workload,VMATMC20} this resource by splitting it among the users (or teams) independent of their demands. It effectively creates siloed pools, it gives strong availability isolation, and it upholds fairness. However, crucially, it doesn't satisfy Pareto efficiency because it ignores demands. Max-min fairness~\cite{DRF11,grandl2016altruistic,hong2013achieving,jain2013b4,narayanan2021solving,narayanan2020heterogeneityaware, popa2012faircloud, pu2016fairride, shue2012performance} addresses this problem, and simultaneously achieves Pareto efficiency and fairness, by incorporating demand into decision-making: it maximizes the minimum allocation across all participants, but subject to not allocating more than their demand. But the efficiency guarantees of max-min fairness mechanisms heavily depends on the assumption that user demand is static over time, which is unrealistic for most real world systems~\cite{VFACKT23}. In settings with dynamic demand, static allocations are bound to be far from Pareto efficient.

\emph{Chip-hour mechanisms.} 
Aiming for more efficient allocations in the presence of dynamic user demands, ``chip-hour'' mechanisms determine the amount of access to the resources that each user should be allocated, but provide the users with some flexibility regarding when to use their access rights.
A notable example is the Karma mechanism, introduced by \citet{VFACKT23}, which asks the users to report the amount of resources that they need in each round, and then determines how the resources will be allocated as a function of each user's access history; specifically, this mechanisms strictly prioritizes users who have had the least amount of access to the resources in the past. This elegant mechanism provides incentives for users to request a resource only if they value it, it allocates resources only to users who need them, and it strives to optimize max-min fairness by prioritizing the user with the minimum amount of resource access in the past.
However, despite the very significant improvements that this mechanism achieves, it has an important limitation: it assumes that the user's value for each of these requests is uniform over time and across users, which is quite unrealistic. In fact, as we show in Section~\ref{sec:theory}, the appealing properties of all chip-hour mechanisms break down even for bi-valued instances, where each agent's value for a request can take just one of two values, say \textsc{High} or \textsc{Low}. Different workloads necessarily hold different values in an organizational context, and it is vital for a resource allocation mechanism to accommodate this. Furthermore, chip-hours are still zero-sum quantities, like chips themselves, because they are still tied to a specific resource type at a specific location. Allocating more chip hours to one team necessarily involves reclaiming these chip hours from some other team(s)\footnote{One can resort to oversubscription, but of course that has it limits too, and the core point that chip hours are zero-sum quantities remains the same.}.


\paragraph{Our Contributions.} Our main contributions are:
\begin{enumerate}[leftmargin=13pt]
    \item We describe the design, implementation and deployment of Quota Marketplace in Google. The deployed system has several thousands of daily active users, and several billions of accelerator hours already consumed via the system. 
    \item We provide metrics from the actual deployment of QM.
    \item We provide a theoretical analysis in a simplified single-resource setting which exhibits the shortcomings of the chip-hours mechanisms, especially in the presence of heterogeneous user values (an inevitable characteristic of workloads in realistic settings). We analyze our market-based mechanism in the same setting, and show that it guarantees Pareto efficiency and approximate max-min fairness, even in settings with heterogeneous values.
    \item Finally, we present several avenues for future work, some that require solving new problems, and others that require careful UX design to ensure that the system does not compromise user-friendliness while aiming to be robust to manipulations.  We also describe how we address some of the challenges raised in prior work (e.g.~\cite{SNPASVC05}) regarding the integration of markets and systems.
\end{enumerate}

While several benefits of QM have already been discussed earlier, we describe a few more in the paragraphs below.

\paragraph{The Power of Dynamic Pricing.} The primary shortcoming of prior works in the face of heterogeneity in values is that they don't provide users a way to express their utility (i.e., their value for being allocated the resource at each point in time), so every user request is treated with the same priority. In the presence of heterogeneous values, we instead need a mechanism which allows users to express these values and then allocate the available resources accordingly. This is precisely what the Quota Marketplace achieves. Given that real-time prices for different resource types are published in a company-wide dashboard, deciding to not start a workload when the price exceeds a certain amount is a definitive way to express that one's value for that workload is below a certain price. Further as already discussed, prices provide a clear incentive for users to shift their demands towards cheaper (and, hence, less congested) resources that they still highly value. 

\paragraph{The Benefits of the Credit Abstraction.}
The role of the credits allotted to the users in our market-based mechanism and their use toward dynamic resource pricing, provide some important benefits relative to static allocation and chip-hour mechanisms. While we briefly alluded to these in the earlier discussion in various contexts, we collect them all below in a little more detail.

First, \emph{credits are fully decoupled from physical resources}: indeed, the sum of credits in circulation by our mechanism is not constrained in any way by the number of available physical resources. This property enables the market to effortlessly handle supply additions and clawbacks: \emph{the addition or clawback of hundreds of thousands of chips does not require the market to adjust the agents' credit balance in any way}:
it simply runs the same market clearance algorithm in the next cycle, using the updated supply value when equating supply and demand to compute clearing price. There is no equivalent to this in static allocations or chip-hour mechanisms.

Second, very related to first, but subtly different, \emph{credits are not zero-sum quantities} --- this directly enables the market to reprioritize in minutes. For instance, granting a certain team or business unit additional power simply requires minting more credits and handing it over to that team, or increasing the market weight of that BU. In contrast, \emph{chips or chip-hours are both zero-sum} since the physical resources they are tied to are, of course, zero sum. As a result, reprioritization to grant additional power to one team necessarily involves determining which teams to withdraw supply from.
 
Third, \emph{credits are agnostic to the type, location, or other attributes of each resource, whereas chips and chip-hours are tied to a particular chip and location}. Therefore, when allocating credits, the allocation committee deals with a single-dimensional problem of deciding who gets how much purchasing power. Allocating chips or chip hours, on the other hand, requires a separate resource allocation decision for each chip type and location. Further, it is easy for the allocation committee to measure the return on investment from different teams based on a single-dimensional quantity of credits. Otherwise, the committee has to wrestle with the very hard problem of comparing apples and oranges when evaluating the impact of one team that was allocated one set of chip types in some set of locations relative to another team that received a different set of chip types and locations.

\paragraph{Organization of the Paper.} 
In Section~\ref{sec:system},  we give a formal description of the Quota Marketplace system and its implementation, along with the features in the market that prevent market abuse. Section~\ref{sec:empirical} presents metrics from the deployment of the Quota Marketplace and demonstrates how the desiderata that we have described thus far are satisfied by the market-based mechanism. 
In Section~\ref{sec:theory} we provide our theoretical analysis. In Section~\ref{sec:challenges} we discuss some of the old and new challenges, as well as new opportunities and interesting open problems for future research, and in Section~\ref{sec:related} we conclude with an overview of related work.

%% file: system.tex
\section{Quota Marketplace System}\label{sec:system}
The Quota Marketplace is an iterative resource allocation system based on dynamic pricing using a generated currency of "credits". To describe how this market operates, we first provide some core concepts.

\begin{figure}[h!]
    \centering
    \includegraphics[width=0.75\linewidth]{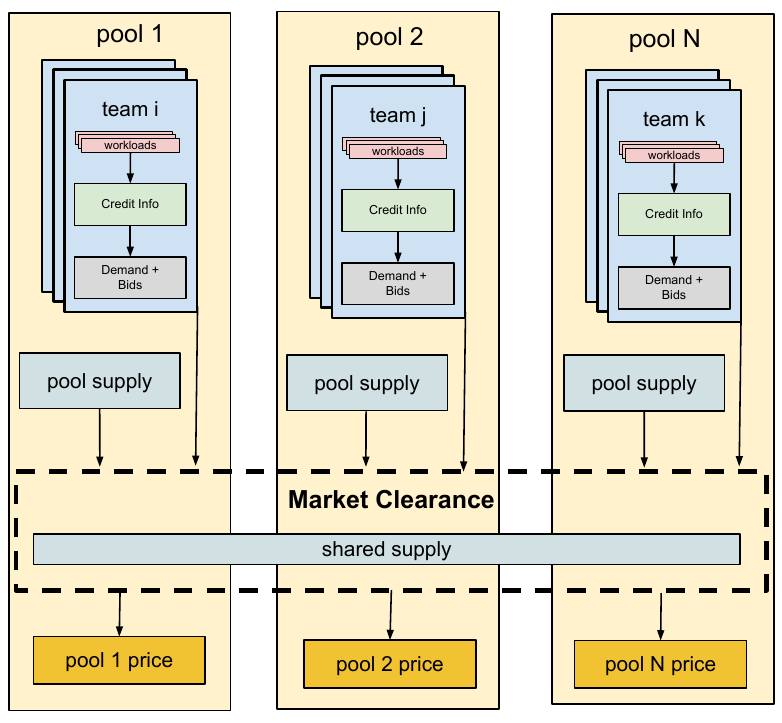}
    \caption{The high-level concepts of Quota Marketplace}
    \label{fig:gqm_system_high_level}
\end{figure}

\paragraph{Supply.} Supply (or resources) of the marketplace correspond to specialized ML-accelerator chips. The chips are distributed across various geographic locations (which we refer to as cells) and are of different generations (which we refer to as resource types or chip types). Each cell has a certain capacity of each chip type, which varies over time.

While the intention of QM is to manage supply shared across all pools, for legacy reasons the system also supports pool-specific supply, which is supply previously purchased and owned by each business unit to which this unit has preferential access. Indeed, Figure~\ref{fig:gqm_system_high_level} shows that market clearance takes into account both pool supply and shared supply.

\paragraph{Teams.} Users are grouped into teams and submit demand to the system in the form of individual workloads. As workloads are submitted, they enter the team priority queue, allowing for relative prioritization.  The team is the fundamental unit within the system to which resources are allocated.

\paragraph{Pools.} Pools are groups of teams representing business units within the larger organization with their own goals, cost accounting, governance, and resources.

\paragraph{Income and Balance.} 
Pool administrators assign credit income to teams to reflect their relative priority in the pool. Income is defined as credits to be deposited into a team's account at a certain rate.
As teams incur charges for resource usage, their credit balance gets reduced. Any newly assigned income (say for a new team, or increased income for an existing team) increases the rate of minting new credits, thereby increasing the credits available within the pool.

Each pool independently manages its credit allocation much like a country manages its own currency. However, the total purchasing power of a pool remains fixed regardless of the number of credits minted within the pool. This invariance is enforced by a mechanism called \textit{market weights}.

\paragraph{Market Weights.} 
Google adjusts priority between pools by assigning market weights, which correspond to the total purchasing power of each business unit. This is implemented by mapping the local credits of each pool to a standard global unit, similar to currency conversion in international trade.

Formally, if a pool $k$ with an assigned market weight of $W_k$, mints an sincome of $C_k$ credits per unit time, each local credit owned by pool $k$ is valued at $W_k / C_k$ standard global units.

This design decouples global prioritization (managed by Google-level administrators via $W_k$) from intra-pool allocation (managed by local administrators via $C_k$). However, this exposes teams to monetary dynamics akin to those in international economics. E.g., if pool administrators mint excessive credits (increasing $C_k$) while the market weight $W_k$ remains constant, the pool experiences inflation and credit devaluation.



\paragraph{User Journey: Reducing Users' Cognitive Burden via Automated Bidding.}
Incentivizing users to change their behavior necessarily entails imposing at least a small amount of cognitive awareness from the users. 
That said, QM is designed to offer most of its incentive benefits with very little investment of time from users. 
A user's journey involves starting a job, using exactly the same workflow as they did in the pre-QM world. Next, the user places their job in a priority queue of active jobs from their team, with the possibility of being asked to adjust their job's position by their team admin, if they so feel. Once a job is on the queue, the most important task of submitting a bid for the job is handled by an automated bidder. The system allows users to set team-level settings and job-level settings for bidding. These include (a) limit orders, which is a job-level setting, letting users specify the maximum price they are willing to pay for that job (b) income-leverage, which is a team-level setting that upper bounds the maximum bid allowed from a team at the product of the team's income and the income-leverage (c) a tool that lets team admins to configure policies for their team's jobs, from a simple priority queue to order their jobs to more sophisticated policies to dynamically manage importance of their jobs. Users throughout the company have access to a dashboard with real-time prices of all chip types, price history for the past several months, their team's credit balance, income rate, credit burning rate to help them make informed decisions while deciding on these settings. The automated bidder descends the priority queue for each team, estimates the cost of running each each job, and then, governed by the settings, aggregates costs to compute team level bids. 
In practice, system provided defaults work quite well for the settings, and indeed most users rely on them, except in periods of high price volatility, which we explicitly mitigate as described in later sections (see ``Volatility and Churn'' in Section~\ref{sec:MarkImp} for example). 

\paragraph{Demand and Bid Formulation.}
In every market cycle, the system aggregates the demand of queued workloads for each team. We denote $d^r_i$ as the aggregate demand of team $i$ for chips of type $r$, where $r$ encodes both the hardware generation and the specific cell. The automated bidder translates the team's balance, income rate, and settings into a scalar bid $b^r_i$ according to relative workload priority. This value represents the total credits the team is willing to spend to secure resource $r$ in the current cycle. Because the system resolves markets per resource type, the agent automatically proportions the team's total buying power across the relevant markets before clearing occurs.

\paragraph{Market Clearing.} 
Let $b_i^r$ and $d_i^r$ denote the bid and demand, respectively, of team $i$ for resource type $r$. For each pool $k$, we compute a specific clearing price $p_k^r$ to equate the pool's internal demand with its available supply. Pool prices differ from the global market price because pools effectively blend their own static resources (which are free to the pool) with resources purchased from the shared market.

For clarity, we assume all bids are normalized to standard credit units, i.e., we ignore currency conversion rates, as it doesn't change the equations much except for adding notational clutter. Let $s_k^r$ denote the static supply owned by pool $k$, and let $z_k^r$ denote the supply purchased by the pool from the shared market supply at a reference price $p_*^r$. The pool sets its internal price $p_k^r$ so that it breaks-even: i.e., the credits it generates from its teams for the $s_k^r + z_k^r$ chips they consume is equal to the credits it has to pay to the central market for having purchased $z_k^r$ credits from it. This yields the following break-even condition:

\begin{equation}\label{eq:blended_price} \tag{1}
p_k^r (s_k^r + z_k^r) = p_*^r z_k^r \quad \implies \quad p_k^r = p_*^r \frac{z_k^r}{s_k^r + z_k^r}
\end{equation}

Equation \eqref{eq:blended_price} highlights the discount effect: if a pool relies entirely on shared resources ($s_k^r=0$), the pool price equals the reference price ($p_k^r = p_*^r$). Conversely, if a pool relies heavily on owned supply, the internal price is significantly lower than the market reference price.

Given a pool price $p_k^r$, team $i$ demands $\min(d_i^r, b_i^r / p_k^r)$ units, i.e., a team's demand is simply the minimum of what it actually needs ($d_i^r$) and what it can afford at a given bid ($b_i^r/p_k^r$). To clear the pool-local market, the total demand from all teams $i$ in pool $k$ must equal the total available supply:
\begin{equation}\label{eq:local_clearing} \tag{2}
\sum_{i \in \text{Pool } k} \min\left(d_i^r, \frac{b_i^r}{p_k^r}\right) = s_k^r + z_k^r
\end{equation}
For any fixed reference price $p_*^r$, Equations \eqref{eq:blended_price} and \eqref{eq:local_clearing} have just two unknowns, namely $p_k^r$ and $z_k^r$, or more technically $p_k^r(p^r_*)$ and $z_k^r(p^r_*)$ as these are functions of $p^r_*$. 
Intuitively, this solution represents the equilibrium between the pool's internal supply-demand clearance and its blended supply cost. Equation \eqref{eq:local_clearing} is the pool's internal supply-demand clearance, where the LHS is the pool's \textit{aggregate demand curve}, mapping the internal price $p_k^r$ to the total quantity demanded by teams, and the RHS is the pool's total supply arising from it's own supply and purchased supply. Equation \eqref{eq:blended_price} represents the \textit{cost-blending curve}, which dictates how the internal price $p_k^r$ must rise as the pool purchases more shared supply $z_k^r$. The intersection of these two curves determines the unique purchase volume $z_k^r$ and internal price $p_k^r$ for that specific market reference price.

Finally, we are left with determining $p^r_*$. Let $Z^r$ be the total capacity of resource $r$ available in the shared marketplace. We seek the lowest price $p_*^r$ such that the aggregate demand from all pools does not exceed the shared supply:
\begin{equation}\label{eq:market_clearing} \tag{3}
 \sum_k z_k^r(p_*^r) \leq Z^r 
\end{equation}

Since the pool demand $z_k^r(p_*^r)$ is monotonically decreasing with respect to price $p^r_*$, the price $p_*^r$ is efficiently computable using binary search. If demand is low (i.e., even at zero prices there is enough supply: $\sum_k z_k^r(0) \leq Z^r$), the price $p^r_*$ is set to zero.

The algorithm described is implementable by a nested binary search, where the inner binary search is used to solve equations \eqref{eq:blended_price} and \eqref{eq:local_clearing} and the outer level to solve equation \eqref{eq:market_clearing}. With the right pre-processing, we implement the algorithm in nearly-linear time using the algorithmic technique in \cite{paes2016field}.

\paragraph{Resource Allocation.} 
Once the specific pool price $p_k^r$ is established, the allocation $x_i^r$ for each team $i$ is calculated as:
$$x_i^r = \min\left(d_i^r, \frac{b_i^r }{ p_k^r}\right)$$
This allocated quota is subsequently passed to the ML scheduler for execution.

\paragraph{Charges and Capacity Import}
In separate cycles from the market, teams are charged for the resources that they utilize and capacity is read from the resource management system. In the charge cycle, credits are deducted according to the resources that the team actually occupied since the last charge cycle. Capacity for all resource types read are persisted for the market to read in each cycle. Both of these cycles are run  less frequently than the market cycle.

\subsection{Market Implementation}\label{sec:MarkImp}
The Quota Marketplace is implemented as a global coordinator layered atop the hierarchical scheduler stack (comprising above-cell and cell-level schedulers). The system operates as a monolithic binary replicated for availability. While the architecture could theoretically support sharding—distributed bidding across teams and distributed clearing across resource types—we opted for a centralized design to prioritize simplicity and coherency in the initial deployment.

Centralizing the market logic in a single process introduces a potential single point of failure with a large blast radius. We mitigate this risk by decoupling quota generation from job execution. QM does not directly schedule jobs; instead, it generates quota inputs for the underlying scheduler. In the event of an outage, the system gracefully degrades by reverting to static pool scheduling using the last valid quota snapshot calculated by the market. Consequently, while the scheduler may rely more heavily on opportunistic rather than prioritized allocation—reducing efficiency—the system avoids catastrophic downtime.

\paragraph{Modular Design.}
As illustrated in Figure~\ref{fig:QMSystem}, QM employs a modular architecture where distinct phases of the market cycle are isolated via strict interfaces. While the system communicates with external components—specifically the scheduler and the resource management system—via standard Remote Procedure Calls (RPCs), internal interactions utilize the same framework primitives without incurring the overhead of network serialization. This design choice allows us to inherit robust distributed systems capabilities, such as deadline propagation, request cancellation, structured logging, and monitoring, within a unified process. Figure 3 highlights these external boundaries: the system ingests demand signals from the above-cell scheduler (\texttt{SampleDemand()}) and physical capacity data from the resource management system (\texttt{ImportCapacity()}), subsequently pushing allocated quotas back to the scheduler (\texttt{UpdateSchedulerInput()}).


\paragraph{Periodic Routines.}
The QM system executes three independent periodic routines: the market cycle, accounting (processing charges and income), and capacity ingestion. Each routine executes within a strict periodic closure that enforces serial execution, preventing overrun. The execution frequencies are tuned to the volatility of the data: the market cycle runs at high frequency ($<1$ minute), income and charge accounting runs every 5 minutes, and capacity ingestion occurs roughly hourly ($O(h)$). All critical system state—credits, capacity, quota, and demand—is persisted in a strongly consistent datastore. Market cycles are keyed by timestamp to ensure consistency across potentially overlapping executions.

By decoupling charging, income generation, and capacity import from the market, the system is able to react much more rapidly to demand changes. Capacity is relatively static on the order of minutes (emergency capacity changes can still be manually propagated in $O(m)$). Income and charge resolution indeed affect bidding power each cycle. But there is a feature of the market called \emph{minimum affordable duration} that ensures that the 5 minutes frequency of processing income and charges is more than enough to make bids fluctuate quickly \emph{only because of demand fluctuations}. Minimum affordable duration is a measure we introduced to ensure market stability and prevent abuse. It enforces that a team can bid an amount only if the burn rate of credits at that bid is sustainable for the next 120 minutes. Therefore, income and charges not being processed for the next 5 minutes will not affect the bidding power of any team. 

\paragraph{Replayable Market.}
The core market clearing logic is encapsulated as a stateless, pure function, \texttt{RunAuction}. All function inputs and outputs are persisted alongside auxiliary system data. Isolating the auction logic from the runtime environment enables deterministic offline simulation and replay of historical iterations using the exact same RPC handlers. This architecture has proven indispensable for reproducing production incidents, iterating on mechanism design, and verifying the correctness of changes.

\paragraph{Deployment and Verification.}
The monolithic binary architecture complicates standard canary deployment strategies. To compensate, we rely heavily on offline simulation and deterministic replay. Furthermore, although the binary is deployed globally, the market is logically sharded by resource type. This allows us to gate new features or configuration changes incrementally per resource type, enabling granular rollouts despite the unified deployment artifact.

To prevent regressions and ensure high-level correctness, we enforce a suite of invariant checks on every change. These checks leverage the market's replayability to compare test outputs and key metrics against historical baselines.

\begin{figure}[h!]
    \centering
    \includegraphics[width=1\linewidth]{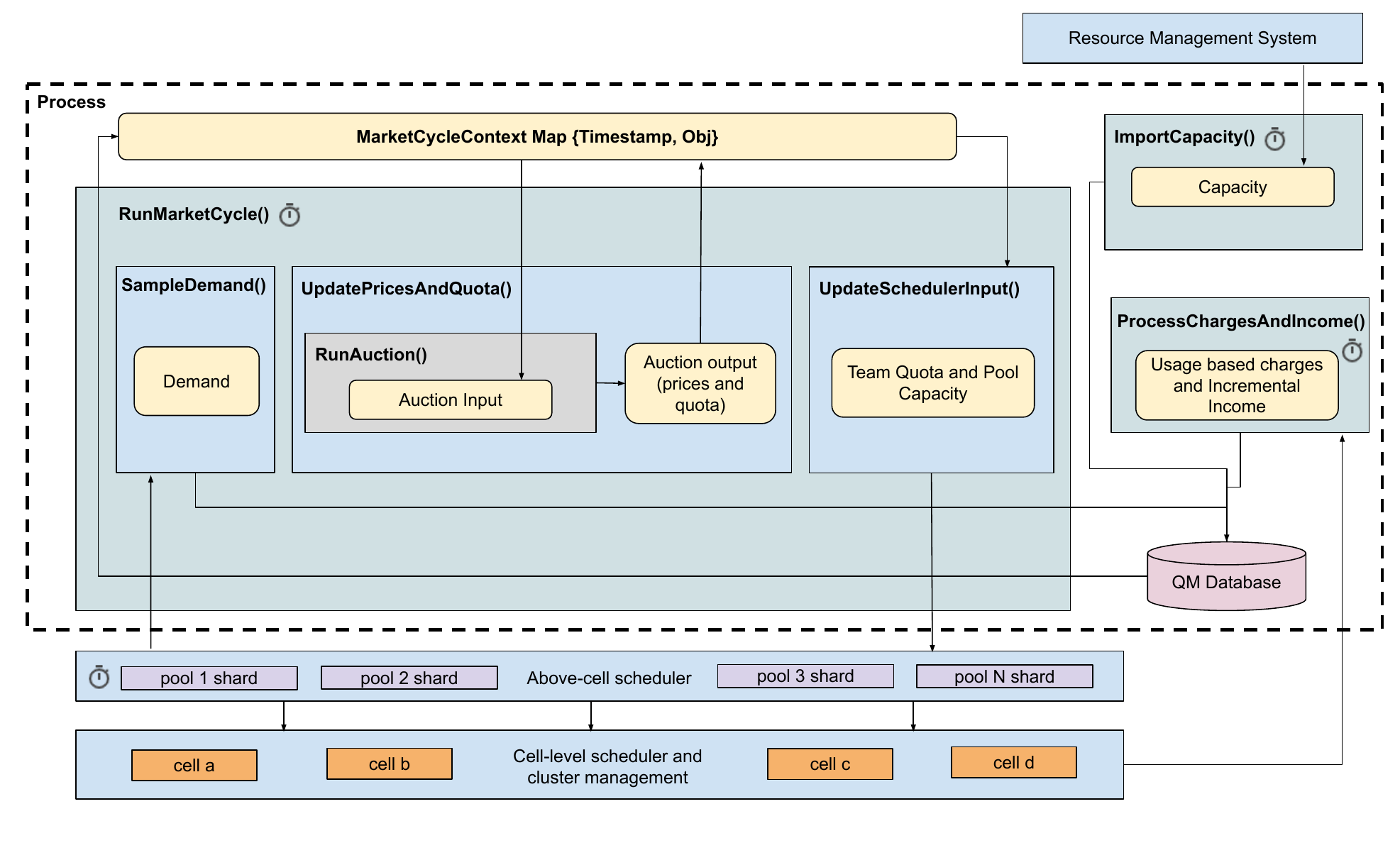}
    \caption{Architecture and components of the Quota Marketplace}
    \label{fig:QMSystem}
\end{figure}


\paragraph{Spatial Flexibility and Global Quota}
The system provides a location abstraction allowing users to submit \textit{spatially flexible} workloads (i.e., workloads that can run in any cell). We initially considered sharding the market by both resource type and cell. However, our analysis revealed that supply at the cell level is highly fragmented; running independent cell-level auctions yields ``thin'' markets with highly volatile pricing and unstable scheduling outcomes (smaller cells averaged $\sim$3\% poorer occupancy than larger ones). To maximize market liquidity, we aggregate demand and supply into a single \textit{Global Cell} abstraction, run a global market per resource type, and allocate global quota.

This abstraction introduces a trade-off between market stability and placement granularity. Global quota guarantees that capacity exists \textit{somewhere} in the fleet, but not necessarily where to a specific workload's is location constrained. Consequently, in roughly about 10\% of the cases, we observe cell-level contention despite being allocated sufficient quota in the global market. We accept this reduction in theoretical market efficiency as a necessary cost to ensure price stability and a usable experience for engineers.

\paragraph{Scheduler Interaction and Sharded Pool Scheduling.}
QM began operating as a naive asynchronous coordinator atop a hierarchical scheduler. It published quota and capacity updates, which the \textit{Above-Cell Scheduler} consumed to place logical workloads in some cell, followed by the \textit{Cell-Level Scheduler} which performed physical bin-packing of the workload in the chosen cell. Because these layers operated asynchronously, state conflicts were inevitable. We deliberately chose this loosely coupled design over deep integration to accelerate the initial deployment and isolate the market logic from the critical path of the scheduler kernel. While this introduces consistency challenges, it allowed us to iterate on the market mechanism without destabilizing the underlying cluster management stack. 

Relatedly, due to legacy constraints, the above-cell scheduler made separate scheduling decisions for each pool. This forced the QM system to physically move capacity between pools so that the above-cell scheduler can realize the quota allocated by the market. This operational lag between the market (which allocates quota immediately) and the resource manager (which moves capacity slowly), compounded mismatches between workload constraints (e.g. global quota vs. location constraints discussed earlier) and available supply. 

We very recently mitigated these issues by moving up the above-cell scheduling logic into the market system and thereby unifying scheduling across pools. This removes the asynchrony between the market and the scheduler and removes the expressivity limitations imposed by the quota interface. This change allows the market to define the inevitable contended cell queuing behavior described in the section above. Importantly, it also paves the way for the market to have a variety of workload-level signals, which it didn't have earlier. 
This architectural shift has major current and future benefits, but comes with a tradeoff: it trades off increased failure domain risk (centralization).

\paragraph{Market Cycle Latency Improvements.}
The latency of the market cycle determines the system's control loop frequency and its ability to react to demand spikes. Through optimizations in the auction algorithm and RPC handling, we reduced the p50 market cycle latency from $100s$ to under $30s$. End-to-end scheduling latency acts as a ``tax'' on the system, resulting in a measurable 1.3\% occupancy wastage across the fleet. As another significant benefit of the architectural shift discussed in the preceding paragraph, we were able to reduce the end-to-end latency to approximately $30s$. This recovers some of the occupancy wastage and improves engineering velocity by reducing the time between workload submission and feedback.

\paragraph{Volatility and Churn.}
Because QM adjusts quotas dynamically based on price fluctuations, it naturally incurs somewhat higher preemption rates than static systems. To prevent high-frequency volatility, as previously mentioned, we enforce that teams must demonstrate sufficient credit balance to sustain a bid for a minimum duration. This smoothing mechanism ensures that bids do not oscillate with the granular income and charge cycles. We treat market stability not merely as an optimization but as a critical usability requirement; users rely on stable price signals to interpret scarcity effectively. Beyond this, there are a host of other scheduling algorithm features designed to mitigate scheduling churn.

\paragraph{Workload Boundaries and Cube Aware Scheduling.}
While the market models demand as a scalar fluid, physical ML workloads are discrete and topological (e.g., a 64-chip request often requires a specific $4 \times 4 \times 4$ cube, corresponding to 16 interconnected machines). The cell-level scheduler handles the physical bin-packing and defragmentation of these shapes. QM is currently \textit{topology-oblivious}: it assigns aggregate quota without guaranteeing that the free chips form a contiguous cube. This abstraction gap inevitably leads to defragmentation-based preemptions where quota exists but geometry does not, and today it arises roughly in about 1\% of the cases.  For the initial deployment, we accepted this inefficiency to avoid the complexity of solving a multi-dimensional knapsack problem in the market loop. Future iterations aim to incorporate topology-aware pricing to bridge this gap.

\paragraph{Auxiliary Resource Requirements.}
QM employs a \textit{principal resource pricing} model: it explicitly prices only ML accelerators, while auxiliary resources (CPU, RAM, disk, network) are allocated proportionally to the accelerator quota. This design is grounded in the observation that accelerators are the dominant bottleneck resource for ML training at Google. Further, this design choice avoids a combinatorial marketplace, and presents users with a single-dimensional price signal, significantly reducing the cognitive load required to participate in the market. While combinatorial auctions could theoretically offer higher allocative efficiency, we prioritize the operability of a simpler pricing model.

%% file: empirical.tex
\section{Deployment Metrics and Takeaways}\label{sec:empirical}
In this section, we quantify the gains unlocked by the deployment of Quota Marketplace, Google-wide. There are many dimensions along which we discuss improvements. For some cases, the comparison we present is with the conventional static pools. In other cases, when there are entirely new opportunities unlocked by the Quota Marketplace, we discuss them with metrics.

\paragraph{Occupancy \& Company-Prioritized Occupancy.} In the Introduction, we described how QM would increase both company-prioritized occupancy and overall occupancy. Here, we quantify these increases.
Figures~\ref{fig:gqm_prod_opportunistic} and~\ref{fig:nongqm_prod_opportunistic} illustrate this increase by showing occupancy fractions of prioritized and opportunistic allocations of the same, very popular, resource type in two different pools: Quota Marketplace and conventional, respectively. Both pools are comparable in their sizes with hundreds of thousands of chips each. The difference in occupancies is stark. The average overall occupancy (i.e., prioritized + opportunistic) in the conventional pool is at 75\%, and the same quantity in the Quota Marketplace pool is at 93\%. This marks a ${\sim}24\%$ increase in overall occupancy. Also, average opportunistic occupancy in the Quota Marketplace pool is at a small 5.67\%, as opposed to 33.5\% in the conventional pool. This fully validates our expectation that the market-based system would substantially eliminate opportunistic occupancy, and what remains of it would be small, arising from some system inefficiencies.

\begin{figure}[h]
\includegraphics[width=0.5\textwidth]{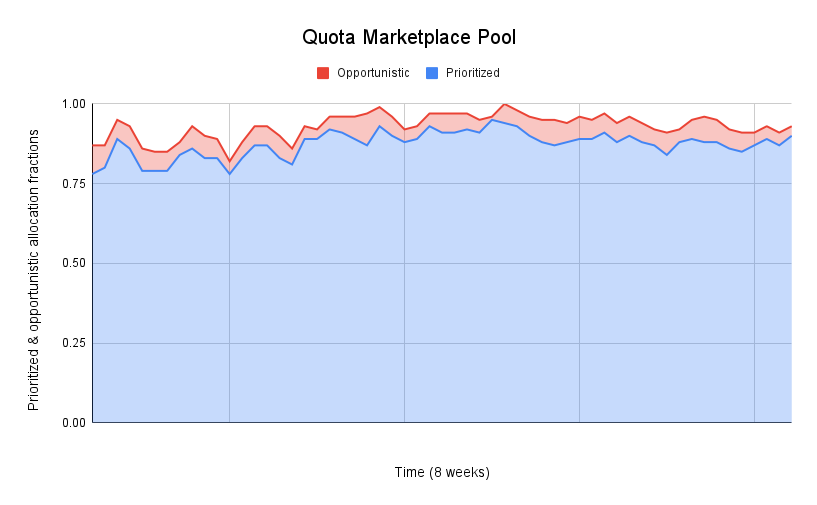}
\caption{Occupancy fractions of company-prioritized and opportunistic allocations in the Quota Marketplace pool for a popular resource type.}
\label{fig:gqm_prod_opportunistic}
\end{figure}
\begin{figure}[h]
\includegraphics[width=0.5\textwidth]{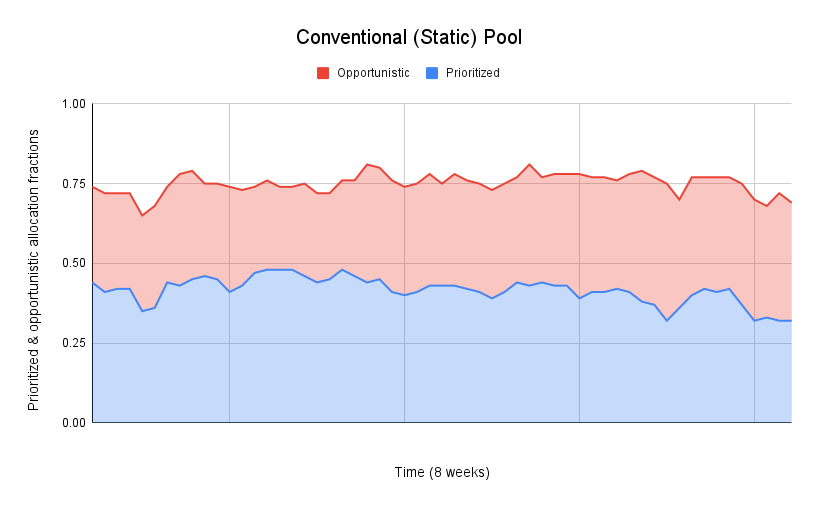}
\caption{Occupancy fractions of company-prioritized and opportunistic allocations in the conventional (static) pool for the same popular resource type as in Figure~\ref{fig:gqm_prod_opportunistic}.}
\label{fig:nongqm_prod_opportunistic}
\end{figure}

\paragraph{Time Varying Supply.} The available pool of compute resources often experiences significant fluctuations driven by evolving business priorities and supply chain dynamics. Common events include supply reductions via clawbacks for reprioritization, temporary capacity injections, reallocation of resources across different geographical locations or accelerator types, and changes to anticipated delivery timelines for new hardware. These frequent supply adjustments pose a challenge to traditional resource allocation systems, often resulting in slower user adaptation and, consequently, suboptimal fleet utilization. Figure~\ref{fig:supply_updates} shows that supply variations over time are substantial, and are not just a theoretical worst-case situation to reason about. While the figure captures shifts in accelerator types, note that supply fluctuations extend beyond this. For instance, even within a single accelerator type, numerous changes in geographical location (among the 3-digit number of data centers) can occur.


\begin{figure}[h]
\includegraphics[width=0.5\textwidth]{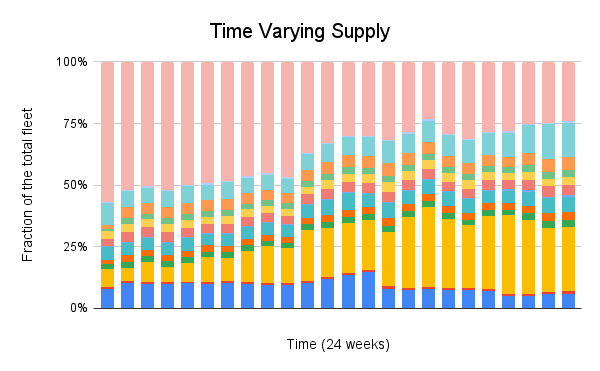}
\caption{The total fleet managed by Quota Marketplace broken down by resource type over a period of 24 weeks.}
\label{fig:supply_updates}
\end{figure}

\paragraph{Minting Supply from Transient Resources.} 
As discussed in the Introduction, historically, in order to maintain allocation speed, various processes had to rely on substantial amounts of readily available buffer resources, that were underutilized and unprioritized in the meantime. Since the Quota Marketplace allows clawbacks in a matter of minutes, this enabled the addition of a large fraction of these buffer resources to Quota Marketplace, resulting in additional minted supply. We club multiple such buffer sources into an umbrella category named ``serving buffer''. Likewise, new capacity keeps arriving on a daily basis, even though the promised delivery date of these chips could be several days or weeks down the line. This is because fulfilment of a large consignment necessarily has to happen in batches, slowly building up to the deadline where the entire consignment is due. On the other hand, even if one were to allocate these chips that arrived ahead of time to the business units that are due to receive them later, the business units cannot be charged back ahead of their scheduled start date. Therefore, these soon-to-be assigned chips can neither be charged back for, nor be put to other high priority uses (given that clawbacks are time-consuming in traditional pools). Such temporarily available resources tend to be underutilized and unprioritized in traditional systems, but they can be readily utilized by the Quota Marketplace. We call the supply that the Quota Marketplace receives from such sources as the ``holding pool''. Figures~\ref{fig:committed_and_bonus_capacities} and~\ref{fig:percentage_lift_from_bonus_capacities} serve to illustrate the magnitude of bonus capacities unlocked by the Quota Marketplace from serving buffer and the holding pool. The bonus capacity unlocked by the Quota Marketplace consistently is in the order of hundreds of thousands of chips, reflecting the substantial volume of reclaimed resources integrated into the marketplace. Note that given the 8-year period over which demand has been growing tenfold annually, and correspondingly companies have been heavily investing in these chips, these holding pools are inevitably substantial. Thus, a crucial contribution of the Quota Marketplace is its ability to effectively harness large quantities of transiently available resources, improving overall system utilization.



\begin{figure}[h]
\includegraphics[width=0.5\textwidth]{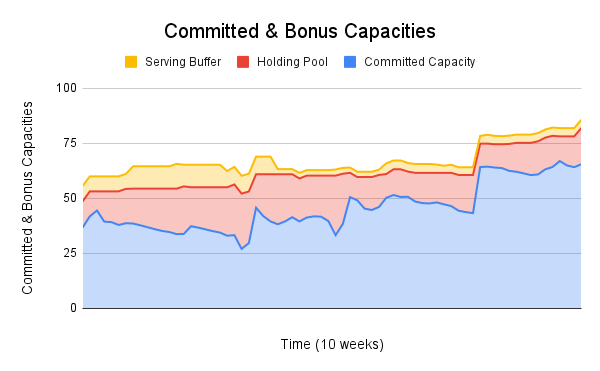}
\caption{Over the period of 10 weeks, we see the bonus capacities that the Quota Marketplace has made available from different sources (y-axis is normalized). The total bonus is substantial, and remains consistently so.}
\label{fig:committed_and_bonus_capacities}
\end{figure}
\begin{figure}[h]
\includegraphics[width=0.5\textwidth]{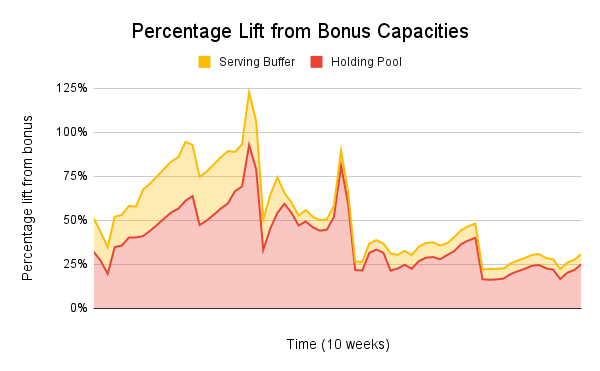}
\caption{Another way to visualize Figure~\ref{fig:committed_and_bonus_capacities}: over the same period of 10 weeks as in Figure~\ref{fig:committed_and_bonus_capacities}, the percentage lift in capacities offered by the two different sources. The lowest point in this 10 weeks period is about 25\%.}
\label{fig:percentage_lift_from_bonus_capacities}
\end{figure}

\paragraph{Incentives in Action: Occupancy.} The operational dynamics of the Quota Marketplace periodically offer empirical insights into user responses to its incentive structures. A key example arises during the regular deployment of new accelerator hardware batches to specific geographic locations or cells. By measuring the time lag between the introduction of this new supply and the corresponding rise in resource occupancy, we can quantify the market's efficiency in absorbing new capacity. We conducted a comparative analysis focusing on this phenomenon. As shown in Figures~\ref{fig:gqm_supply_occupancy} and~\ref{fig:nongqm_supply_occupancy}, two resource pools received simultaneous additions of new accelerators: one operating under the Quota Marketplace framework and a control pool managed via traditional static allocation. Within the Quota Marketplace pool, occupancy levels rapidly tracked the increase in available supply, while in the conventional pool, occupancy substantially lagged supply. Significantly, this swift uptake in the Quota Marketplace pool occurred organically, without requiring broad organizational announcements or specific notifications regarding the new hardware's availability or location.  Users, guided predominantly by the price signals within the marketplace, efficiently discovered and migrated towards these newly-introduced and lower-priced resources.


\begin{figure}[h]
\includegraphics[width=0.5\textwidth]{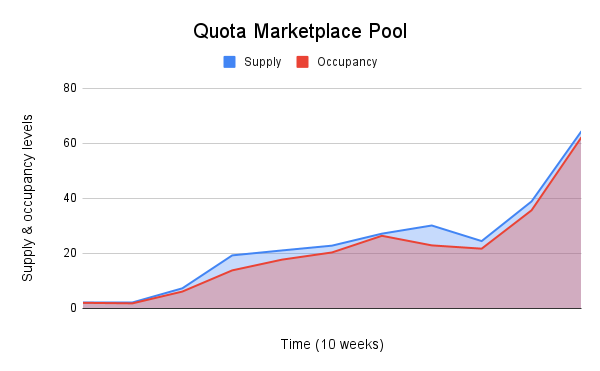}
\caption{Over a 10-week period when a new resource type was introduced, we see how occupancy closely tracks supply, even during substantial supply spikes (y-axis is normalized).}
\label{fig:gqm_supply_occupancy}
\end{figure}

\begin{figure}[h]
\includegraphics[width=0.5\textwidth]{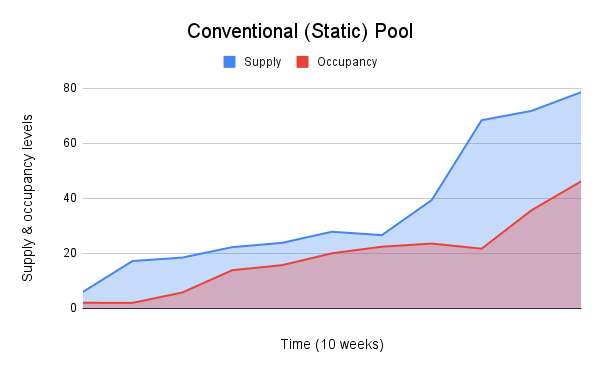}
\caption{Over the same 10-week period as in Figure~\ref{fig:gqm_supply_occupancy}, when the same new resource type was introduced in the conventional static pool, we see how occupancy substantially lags behind supply (y-axis is normalized as in Figure~\ref{fig:gqm_supply_occupancy}).}
\label{fig:nongqm_supply_occupancy}
\end{figure}

\paragraph{Incentives in Action: Demand Shaping.} 

The introduction of a new resource type, that is a partial substitute to a pre-existing resource type, provides another opportunity to obtain useful insights into user responses to incentives. Initially available at very low prices due to lack of demand, this new resource type requires users to adapt their workflows for compatibility. Despite the existence of this friction, the price of this new resource type rapidly increases to almost catch up with the pre-existing resource type. This price convergence occurs organically through user responses to price signals and peer communication via chat and email about the existence of a very low-priced alternative, rather than management directives or mandates. This behavior confirms that the Quota Marketplace effectively shapes demand through pricing mechanisms, i.e., it indeed makes demand flow towards less-congested resource types. 

\begin{figure}[h]
\includegraphics[width=0.5\textwidth]{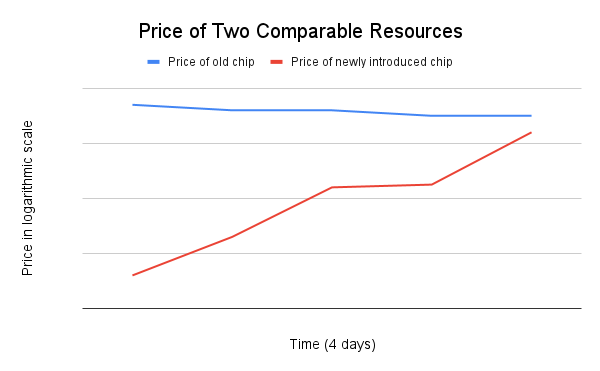}
\caption{When a new resource type is introduced in the market, it starts at a low price, but the price grows exponentially (y-axis is in logarithmic scale) to catch up with that of a comparable pre-existing resource type.}
\label{fig:price_matching}
\end{figure}

\paragraph{Tackling Highly Non-trivial Tasks.}
Another important benefit of the QM and the flexibility that it provides is that it makes it feasible to tackle highly non-trivial large-scale tasks that would have previously required multiple quarters to complete. Such tasks can now be completed much faster (in a matter of weeks) due to the unprecedented acceleration that the Quota Marketplace is able to provide by reprioritizing portions of its massive fleet. Note that even such massive reprioritizations need not severely impact any particular business unit, or team, as their impact is seamlessly amortized across the whole market through price increases.

%% file: challenges.tex

%% file: theory.tex
\section{Theoretical Analysis}\label{sec:theory}
We now formally prove that the market-based mechanism achieves efficiency and fairness guarantees, and we provide examples exhibiting how previously proposed solutions, like Karma, fail to achieve the same. We exhibit this separation in a simple model with a fixed shared supply of a single resource and ignore pool-specific supply. 

\paragraph{Model.}\label{sec:prelim}
A set $N$ of $n$ agents (teams) compete over $s$ units of a single resource, and our goal is to design a fair and efficient mechanism that determines how these units are allocated over time. Specifically, given some time horizon of $T\geq 1$ rounds, the mechanism decides how many units each agent is allocated in each round $t\in [T]$ (where  $[T]=\{1, 2, \dots, T\}$). An allocation $x\in \mathbb{N}^{n\times T}$ determines the number of units $x_{t}[i]$ that each agent $i$ will receive in each round $t\in [T]$ (for simplicity, we assume units are divisible, so $x_{t}[i] \in \mathbb{R}$). For an allocation to be feasible, it needs to allocate no more than $s$ units in each round, i.e., $\sum_{i\in N} x_{t}[i]\leq s$ for all $t\in [T]$.

At each round $t$, each agent $i$ demands $d_{t}[i]\geq 0$ units and has some value $v_{t}[i]\geq 0$ for each of these units. Therefore, $i$'s value for receiving $x_{t}[i]$ units in round $t$, denoted $v_{t}[i]( x_{t}[i])$, is equal to $\min \{x_t[i], d_{t}[i]\}\cdot v_{t}[i]$, and the overall value of $i$ in allocation $x$ is $v_i(x)=\sum_{t\in [T]} v_{t}[i](x_{t}[i])$. This strictly generalizes the class of valuations used in the Karma paper~\cite{VFACKT23}, which correspond to the special case where $v_{t}[i]=1$ for all $i$ and $t$, i.e., all the agent values are uniform over time. Another interesting intermediate class that we consider are bi-valued agent preferences, which are the mild generalization of uniform values where each agent's values can take one of two values, e.g., either \textsc{High} or \textsc{Low}.

\paragraph{Pareto Efficiency and Max-min Fairness.} An allocation $x$ is Pareto-dominated by allocation $x'$ if $v_i(x')\geq v_i(x)$ for every agent $i$ and $v_i(x')>v_i(x)$ for at least one agent $i$. An allocation is \emph{Pareto efficient} if it is not Pareto-dominated by any other feasible allocation. 
If $w_i>0$ captures the relative weight, or significance, of each agent $i\in N$ (often referred to as the agent's ``fair share''), then the (weighted) \emph{max-min fair} allocation assigns to each agent $i$ a fraction $\frac{w_i}{\sum_{j\in N}w_j}$ of the $T\cdot s$ available units. 
We call an allocation $x$ \emph{$\alpha$-fair} if the total number of units allocated to each agent $i$ is at least 
\begin{equation*}
\alpha\cdot \frac{w_i}{\sum_{j\in N}w_j}\cdot T \cdot s. 
\end{equation*}


\paragraph{Market-Based Mechanism.}
Without loss of generality, we assume agent weights are normalized so that $\sum_{i\in N}w_i =s$, and the market-based mechanism starts by allotting a budget of $T\cdot w_i$ credits to each agent $i$.
Then, at the beginning of each round $t$, each agent $i$ reports the number $d_{t}[i]$ of units they request, along with a limit price $\bar{p}_{t}[i]$ (the maximum amount they would be willing to pay per unit), and amount of budget $b_{t}[i]$ that they are willing to spend in that round. 
Given this information and some price $p_t$, the mechanism can compute the demand of each agent $i$ at round $t$ as
\begin{equation*}
D_{t}[i](p_t)=
\begin{cases}
0 & \text{ if } p_t> \bar{p}_{t}[i]\\
\min\{b_{t}[i]/p_t,~ d_{t}[i]\} & \text{ otherwise.}
\end{cases}
\end{equation*}
The total demand at round $t$ for a price $p_t$ is then equal to $D_t^*(p_t)=\sum_{i\in N} D_{t}[i](p_t)$, and the market-clearing price is the maximum price that ensures all units are allocated, i.e., the price $p_t$ that satisfies $D_t^* (p_t)\geq s$ and $D_t^*(p_t+\epsilon)<s$ for any $\epsilon>0$. The mechanism computes this market clearing price and allocates to each agent $i$ a total of $D_{t}[i](p_t)$ units (if $D_t^*(p_t)>s$, the requests that would remain at price $p_t+\epsilon$ are satisfied first, and a tie-breaking rule is used for the rest). The mechanism then subtracts $p_t\cdot x_t[i]$ credits from the total budget of each bidder $i$ and proceeds to the next round.

We make the common ``large market'' assumption that no single agent can significantly affect prices, so they behave as price-takers, and since these are very scarce resources, the number of units requested at each round $t$ is at least $s$, i.e., $\sum_{i\in N}d_t[i]\geq s$. We also assume that agents can estimate their demands and the prices for future rounds, and they choose how much to spend in each round $t$ accordingly. 

\paragraph{
Chip-Hour Mechanisms.} 
To overcome the limitations of traditional approaches that allocate chips in a static fashion (i.e., each agent has a fixed number of chips in every round), Chip-Hour mechanisms allocate ``hours'' of chip access to the agents, and then provide them with some flexibility regarding when to request these hours. In each round $t$, each agent $i$ requests $d_t[i]$ units and, based on what fraction of their chip-hours they have used, and the overall demand, the Chip-Hour mechanism decides how many chips to assign. A prominent example is the Karma system of \citet{VFACKT23}, which strictly prioritizes agents who have used less of their chip-hours so far, aiming to optimize max-min fairness. Some aspects of this system are similar to the market-based mechanism: agents are bootstrapped with some credits and they also earn an income of such credits (in proportion to their chip-hours); they then spend these credits to gain access to resources. However, there is a crucial difference: the ``price'' of every resource in the Karma system is fixed over time, so spending one credit always buys the user access to the same amount of resource. As we show in this section, this limitation restricts the agents' ability to optimize their allocation and it can lead to unfair and inefficient outcomes.

\subsection{Limitations of Chip-Hour Mechanisms}\label{sec:chip-hour-limitations}
To exhibit some limitations of Chip-Hour mechanisms, we first consider the class of uniform valuations (i.e., $v_{t}[i]=1$ for all $i$ and $t$), which was the focus of \citet{VFACKT23} in their analysis of Karma. For uniform valuations, they were able to achieve significant improvements over prior work which assumed that agent requests are static over time (i.e., $d_{t}[i]=d_{t'}[i]$ for any $t,t'\in[T]$). Specifically, they showed that Karma is Pareto efficient for uniform valuations and achieves a ``myopic'' notion of fairness: at any round $t$, given the resource allocation choices made in previous rounds (which may be unfair but cannot be changed anymore), the allocation at round $t$ aims to optimize max-min fairness. 

%
We first exhibit that moving beyond this myopic notion of fairness and aiming for a global notion like max-min fairness is, in fact, impossible, not just for Karma, but for any Chip-Hour mechanism.
This is in contrast to the market-based mechanism which leverages the agents foresight regarding their future demand to move beyond myopic fairness notions and achieve (approximate) max-min fairness even in settings with general heterogeneous valuations (Theorem~\ref{thm:balance}).

\begin{restatable}{theorem}{imbalance}\label{thm:imbalance}
Chip-Hour mechanisms cannot guarantee $\alpha$-fairness for $\alpha>1/n$, even for uniform-value instances.
\end{restatable}
\begin{proof}
Consider a simple instance with $n$ agents of equal significance $w_i =1$, $s=n$ units, and $T=n$ rounds. Agent 1 requests $d_{1}[1]=s$ units in round 1 and $d_{t}[1]=0$ for all subsequent rounds $t>1$. All other agents request $s$ units in every round. The max-min fair solution for this instance would allocate all the $s$ units in the first round to agent $1$ and split the units of the remaining rounds equally among the other agents, allocating $s=n$ units to every $i$. On the other hand, given this instance, and lacking any foresight regarding the drop in future demand from agent $1$, any Chip-Hour mechanism would distribute the units in the first round equally among the agents (since all agents have equal $w_i$ and, hence, equal chip-hours). This yields a single unit ($s/n=1$) for agent 1, leading to an outcome that is only $1/n$-fair.

In comparison, consider the outcome of the market-based mechanism facing this instance. Since agent 1 does not need any units after the first round, she is willing to spend all of her $s$ credits in the first round. This leads to unit price of at least $1$ for the first round. If any other agent $i$ also spent some of their budget on the first round, then at least one round $t>1$ would have a unit price that is strictly less than $1$ (this is because the budget spent on subsequent rounds would be strictly less than $(n-1)s$, while the number of available units in subsequent rounds is $(n-1)s$). Therefore, this would not be a market equilibrium since agent $i$ would prefer to spend on round $t$ rather than round $1$.
As a result, in the market equilibrium every agent receives an equal number of units, leading to the max-min allocation, i.e., $1$-fairness.
\end{proof}

We now show that for the slightly more general class of bi-valued preferences, Chip-Hour mechanisms are unable to guarantee Pareto efficiency too. In fact, moving beyond uniform valuations  introduces incentives for the agents to manipulate these mechanisms and strategically misreport their demands. Our next result shows that their Pareto efficiency is violated, regardless of whether the users behave truthfully or strategically. This is in contrast to the market-based mechanism which achieves Pareto efficiency even for the general class of heterogeneous valuations (Theorem~\ref{thm:efficiency}).

\begin{restatable}{theorem}{NotPareto}\label{thm:Not_Pareto}
Chip-Hour mechanisms violate Pareto efficiency even for bi-valued instances, irrespective of whether the agents behave truthfully or strategically.
\end{restatable}

\begin{proof}
Note that Chip-Hour mechanisms ask each agent $i$ to report their $d_{t}[i]$ demand. We first consider the case where the agents truthfully report this value in each round, and we then also consider the case where the agents may strategically misreport this value (e.g., by reducing the $d_{t}[i]$ in rounds where their value is \textsc{Low}, aiming to save their credits for subsequent rounds of \textsc{High} value). For this proof we normalize the \textsc{High} value to $1$ and the \textsc{Low} value to some $\epsilon\in (0, 1)$.

\emph{Truthful Reporting.} To verify the lack of Pareto efficiency for the case where the agents truthfully report their demand, $d_{t}[i]$, consider the following simple instance involving just two agents and two rounds. In this instance, we have $d_{t}[i]=2$ for both agents and rounds, but the values of agent 1 are $1$ for the first round and some arbitrarily small $\epsilon$ for the second, while the values of agent 2 are $\epsilon$ for the first round and $1$ for the second. In other words, agent 1 strongly prefers the units of the first round while agent 2 strongly prefers those of the second round. However, since Chip-Hour mechanisms do not elicit the agents' values (implicitly assuming that they are uniform over time), the two agents look identical, and the outcome would equally split the two units of each round among the two agents, leading to an overall value of $(1+\epsilon)$ for each. To verify that this is not Pareto efficient, note that if we were to give agent 1 both of the first-round units and give agent 2 the second-round units, then both agents would have a value of $2$ instead, leading to an almost 100\% increase in both values.

\emph{Strategic Reporting.} It is natural to assume that an agent facing a mechanism that disregards the intensity of their demand (as in the instance above), the agent may choose to strategically misreport their $d_{t}[i]$ requests, aiming to maximize their value (e.g., by submitting no requests during rounds of lower value). We now show that even if agents report strategically aiming to optimize their allocation with respect to their underlying values, Chip-Hour mechanisms still fail to achieve Pareto efficiency. We prove this using small instances with two agents, two units, and two rounds. In all instances that we consider, both agents request both units of the second round, so the Chip-Hour mechanism's decision regarding how to allocate those two units is purely a function of the units they were allocated in the previous round (or, equivalently, a function of their remaining budget). We prove that irrespective of how the mechanism uses this information, there exists an instance for which the resulting outcome will be Pareto inefficient.

First, we focus on Chip-Hour mechanisms which prioritize the agents who have used less of their credits, like Karma does. Specifically, for our two-round instance, we consider mechanisms that would allocate both of the second round units to an agent who was not allocated any units in the first round if the other agent was allocated at least one. For this class of Chip-Hour mechanisms, we consider the instance where agent 1 values the first round units $\epsilon<0.5$ and both agents value the second round units $1$. In this case, although agent 1 has a positive value for the two first-round units, he has a higher value for a second round unit, so they would prefer to use no credits in the first round, as this would lose them one (more valuable) unit in the second round. Therefore, the strategic choice of agent 1 would be to report a demand $d_{1}[1]=0$ for the first round and then split the two second-round units with agent 2. This leads to a value of 1 for each agent. However, note that this is not Pareto efficient since the two first-round units remain unallocated even though agent 1 has positive value for them. If we were to instead allocate those first-round units to agent $1$, this would yield a value of $1+2\epsilon$ for this agent (an up to 100\% increase), which Pareto dominates the previous allocation. 

A similar argument can prove the Pareto inefficiency of the (much less natural) class of Chip-Hour mechanisms that instead prioritize agents who have spent more of their credits. In this case, facing the same instance as above, agent 1 would report their true demand, but agent 2 would prefer to lie by requesting the units in the first round (although he has no value for them). This way, he would be allocated one of them (by symmetry) and he would ensure that he is also allocated one unit in the second round (by symmetry). However, this is clearly not Pareto efficient since one of the first-round units that agent 1 has positive value for is allocated to agent 2 who has no value for it.

Finally, we conclude with the class of Chip-Hour mechanisms that would split the second round units equally among the agents irrespective of what happened in the previous round (i.e., they are independent of how many credits each agent has left). For these mechanisms, we consider the same instance that we used for the truthful reporting case above. In this case, no agent has any incentive to misreport (since the mechanism essentially disregards history), so the Pareto inefficient outcome that arises is the same as the one in the truthful case.
\end{proof}

\subsection{Efficiency \& Fairness of Markets}

In contrast to Chip-Hour mechanisms, the market mechanism always guarantees Pareto efficiency, even for instances where the agents have general heterogeneous values. 

\begin{restatable}[Pareto Efficiency]{theorem}{efficiency}\label{thm:efficiency}
The market mechanism achieves Pareto efficiency for general heterogeneous valuations, irrespective of how much the prices may vary.
\end{restatable}
\begin{proof}
Given some instance, let $p$ be the market clearing prices computed by the market mechanism and $x$ be the corresponding allocation that it yields. Assume that the allocation $x$ returned by the market mechanism is Pareto dominated by some other allocation $x'$, i.e., there exists some agent $i$ who strictly prefers their allocation in $x'$ over the one in $x$. This must mean that the total price, according to $p$, of the units that $i$ is allocated in $x'$ must be greater than the units allocated to $i$ in $x$, i.e., $\sum_{t\in[T]}x'_t[i]p_t > \sum_{t\in[T]}x_t[i]p_t$, otherwise $i$ could just buy them and strictly increase their value (this is due to the common assumption that $i$ spends their budget optimally given the price vector $p$). Also, since $x'$ Pareto dominates $x$, it must be the case that every other agent $j\neq i$ is at least as happy in $x'$. Similarly, this implies that $\sum_{t\in[T]}x'_t[j]p_t \geq \sum_{t\in[T]}x_t[j]p_t$, otherwise $j$ must be spending their budget sub-optimally in $x$; specifically, they could instead buy their $x'$ units by spending less than their budget and then spend the remaining budget on some additional units and strictly increase their value. This uses the common assumption that no agent receives everything that they have positive value for, so there are always additional (fractions of) units that they can buy. Summing the aforementioned inequalities over all agents $k$ (both $i$ and $j\neq i$), we get
\begin{align*}
\sum_{k\in N}\sum_{t\in[T]}x'_t[k]p_t &> \sum_{k\in N}\sum_{t\in[T]}x_t[k]p_t\\
\sum_{t\in[T]}\sum_{k\in N}x'_t[k]p_t &> \sum_{t\in[T]}\sum_{k\in N}x_t[k]p_t\\
\sum_{t\in[T]}s\cdot p_t &> \sum_{t\in[T]}s\cdot p_t,
\end{align*}
which is a contradiction.
This result can also be seen as an implication of the First Welfare Theorem.
\end{proof}

Apart from Pareto efficiency, we also show that the allocations induced by the market mechanism are (approximately) fair. Specifically, 
one of the things that we observed though the deployment of the market mechanism is that although the prices of the units can vary significantly over time, the average price that each agent pays per unit in the long run does not vary as much. One way to interpret this is that there is no reason why any one agent's requests may be strongly correlated with higher prices. If we let $\tilde{p}$ denote the average price per unit across the $T$ rounds and let $f_i \cdot \tilde{p}$ be the average price paid by agent $i$ over this time horizon, then we can bound the fairness of the induced allocation using the $f_i$ values. Specifically, the resulting allocation is $\min_{i\in N}\frac{1}{f_i}$-fair, converging to max-min fairness as the $f_i$ values converge to $1$. Note that this is in stark contrast to Chip-Hour mechanisms which cannot even achieve better than $1/n$-fairness when $f_i=1$ (see the instance used in the proof of Theorem~\ref{thm:imbalance}).

\begin{restatable}[Fairness]{theorem}{balanced}\label{thm:balance}
Let $\tilde{p}$ be the average price of a resource over time. If the average price paid by some agent $i$ is at most $f_i \cdot \tilde{p}$, then the resulting allocation is $\min_{i\in N}\frac{1}{f_i}$-fair, even for general heterogeneous valuations.
\end{restatable}
\begin{proof}
Since each agent $i$ is allotted a budget of $T\cdot w_i$ credits, the total number of allotted  credits is $T\cdot s$. Also, the total number of units allocated per round is $s$, so $T\cdot s$ units allocated overall. Therefore, the average price per unit is $\tilde{p}=1$, and the average price paid by any agent $i$ is at most $f_i$. Also, since agent $i$ has a total budget of $w_i \cdot T$, this means that the number of units that they were able to buy using this budget is at least $\frac{w_i T}{f_i}$ overall, which is a $\frac{w_i}{s\cdot f_i }$ fraction of all available units. This implies that every agent receives at least a $\min_{i\in N}\frac{1}{f_i}\cdot \frac{w_i}{\sum_{i\in N}w_i}$ fraction of the $T\cdot s$ available units, so the resulting allocation is $\min_{i\in N}\frac{1}{f_i}$ -fair.  \end{proof}


%% file: future_research.tex
\section{Discussions and Future Directions} \label{sec:challenges}

While the system is deployed in Google to serve thousands of users and manage resources at scale, there remain significant improvements to be explored in maximizing efficiency, flexibility, and user experience. We outline here some unresolved challenges, and  opportunities being currently explored.

\paragraph{Shaping User Behavior.}
Users were shown to respond to market prices which opens the opportunity to use prices as a vehicle for driving other organizational goals. One example is carbon-efficiency: prices can be adjusted to incentivize scheduling in more carbon-efficient datacenters or use more carbon-efficient configurations. Similarly, the market can be changed to provide discounts (or penalties) guiding towards desired behavior. For example, workloads that are well configured to make good use of the hardware, don't crash, save checkpoints can be rewarded. These mechanisms can also be used to drive adoption for workload improvements that require additional user effort or opt-in.

Currently, our system simply balances supply and demand without taking many other factors into account, but the existence of prices creates the possibility of shaping the behavior of users using economics rather than mandates.

\paragraph{Reducing Failure Blast Radius.}
Centralized resource allocation with the QM improves efficiency but also creates a single point of failure. One way to reduce the impact of a bad rollout or bug in the code is to shard the system into multiple deployments, each managing a subsection of the fleet. There are many ways to consider sharding which offer tradeoffs in efficiency and user experience. One approach is to thin out the necessary global component while transferring as much logic as possible to sharded replicas. We can also imagine separating parts the system by location or resource type.

\paragraph{Potential for Gaming.}



Our system operates within Google, where users and teams prioritize their own resource needs but are policy constrained to not intentionally disrupt others. We now briefly discuss a few potential gaming strategies in general adversarial environments. 

\paragraph{Gaming: Currency Conversion.}
When designing QM, we made the conscious decision of letting each business unit mint their own currency. All the business-unit-local credits get converted to a global currency when the market clears. The local-to-global currency exchange rate is a function of both the number of credits minted by that business unit, and the market weight assigned to that business unit by the company. This freedom to mint credits has ergonomic benefits: for example, credits minted per hour, can be maintained in proportion to the number of engineers in the team. However, this freedom also comes with an ability to game the system.  Consider a scenario where a team mints a lot of credits per hour for a while, accumulates a substantial number of credits in its savings, and then dramatically reduces the rate at which it mints credits per hour. The newly minted credits are worth much more because the currency's value had significantly risen. But if the system were to value the accumulated old credits also at the same level as the new credits, the team would have successfully managed to mint credits out of thin air. If we value these credits differently, we introduce complexity.

An inelegant solution for this problem is to declare that each credit comes with the exchange rate at which it was minted, and therefore different credits are worth differently to a team. This creates too much complexity and the meaning of a credit is lost when different credits are worth differently. Another solution with jarring optics is to say that the bank accounts of teams will get depleted by a certain amount (that depends on how much the business unit decided to raise its currency's value), thereby maintaining the same value for all the credits, old and new. But a team's credits vanishing for no choice of theirs (income rate decisions are made by the head of the business unit) is a recipe for many escalations and confusions. Designing good mechanisms that balance simplicity+usability with robustness to manipulations is a useful problem to solve. We consider this as a UX problem and indeed as~\cite{SNPASVC05} notes, the interface has the biggest influence on user perception.

\paragraph{Gaming: Adversarially Depleting Other Teams' Credits.}
Teams bid for demanded resources, but are only charged credits for jobs that run. A bad actor could launch a job which is unlikely to schedule and is set to immediately fail if it does. For example, this job may require all the resources of a specific type in a specific location to start. The bad actor is thus able to inflate prices for all other market participants without paying any credits to do so. 

One approach to disincentivize this behavior is to charge for resource allocations, not usage. However, potentially conflicting scheduling outcomes later in the stack can lead to significant charges for teams that did not get to run their workloads. This becomes less probable as the market becomes more deeply integrated with the scheduler.

Another approach is to use the Vickrey-Clarke-Groves (VCG) mechanism. This would enable the system to price each workload's resources based on the externality it imposes. In other words, each workload's cost would be the difference between the value of the other workloads that would have run in this workload's absence and the value of the other workloads that ran in this workload's presence. A small job which has no measurable effect on the system would be more or less free. A large or hard-to-schedule job, like the one described above, would not be directly displaced by any single workload, so it would not affect their prices. But VCG has its own complications, like having to price each workload differently, as opposed to uniformly pricing every workload that ran at any given point of time at a given location or cell. I.e., VCG's pricing is complicated for users to understand, as opposed to the simple price-per-cell concept that the market guarantees.



%% file: related.tex
\section{Related Work}\label{sec:related}
Systems literature on resource allocation mechanisms is rich and vast.
Here, we focus on a few related streams of work. To the best of our knowledge, ours is the first work that (a) guarantees Pareto efficiency and (approximate) max-min fairness in the presence of dynamic demand with heterogeneous values, and (b) has designed, implemented, and deployed a market-based mechanism for system resources at this scale.

Market-based mechanisms for pricing computer resources is an old idea~\cite{Diamond1968,Sutherland1968,Corbato}. However, pricing for shared resource allocation was not a first-order concern in time-shared systems, leading to this research being short-lived. There have been other works that took a market approach to resource allocation~\cite{Waldspurger1992,Tanenbaum1986}. Two decades ago, when~\cite{SNPASVC05} described why the time was right to revisit market-based solution for systems resource allocation, they cite the demand-supply gap as the primary reason. This has never been any more true than it is today. Therefore, research on market-based ideas for resource allocation is indeed timely. Documented experiences of markets allocating resources in a wide variety of settings are valuable, as they serve as catalysts for wide-spread adoption of this powerful tool. With this motivation, we go over how we addressed some of the markets-systems integration challenges that were raised in~\cite{SNPASVC05}.

\paragraph{Allocation Policy Must be Explicit.}  The main challenge identified is that of landing on an allocation policy that all the users can agree on: whether to maximize allocation efficiency, or to favor jobs from under-represented users etc. In the environment we operate in (a large company) all players are in agreement that they want to maximize the value to the company. Disagreements about what is more valuable can be resolved top-down following the organizational structure -- which mirrors the way credits are distributed.


\paragraph{Calculating and Expressing Valuation.} Despite the significant influence of bidding interfaces on user perception of market mechanisms, this area remains under-researched. Our system allows users to express workload value via parameters specifying job credits. While finer-grained controls are possible, the system must balance expressivity with interface simplicity, as~\citet{SNPASVC05} note that ease of use often surpasses the importance of complex features.

\paragraph{Well-Defined Currency.} \citet{SNPASVC05} also note that virtual currencies have a lot of issues, particularly that they are not exchangeable outside the system. While our credits too cannot be used for buying anything outside the system, we note that business units (not individual teams) are charged back for the market power they are assigned. I.e., business units pay in real dollars for getting credits, and indeed they could have used those dollars towards headcount, other equipment etc. Therefore, credits in QM have a tangible dollar value. Designing a good currency also includes creating mechanisms to control inflation and volatility in prices. QM has guardrails like bid upper bounds to ensure that teams are unable to hoard credits and flood the market with credits.

\paragraph{Pricing-Based Resource Allocation.} Spot-instance marketplace and virtual machine auctions~\cite{amazon_ec2_spot,benyehuda2014ginseng,funaro2016ginseng,wolski2017probabilistic,zheng2016viability,zheng2015how} are examples of pricing-based resource allocation. Nevertheless, fair resource allocation is not their goal, and even in rare cases where fairness was the goal~\cite{wang2015xchange}, it is qualitatively different from the long-term fairness that we focus on~\cite{VFACKT23}. Beyond resource allocation, other places where credits have been used in game-theoretic contexts include~\cite{feigenbaum2004distributed,cox2003samsara,piatek2007do,yadgar2013cooperative}.

\paragraph{Max-min Fairness and Generalizations.} Being able to satisfy Pareto-efficiency and fairness has inspired a lot of work on this topic~\cite{DRF11,grandl2016altruistic,hong2013achieving,jain2013b4,narayanan2021solving,narayanan2020heterogeneityaware, popa2012faircloud, pu2016fairride, shue2012performance}, including generalizations to multiple resources~\cite{DRF11,grandl2014multiresource,grandl2016altruistic}. As described earlier, the efficiency guarantees of these mechanisms crucially rely on the demand being static; the recent Karma mechanism addresses this exact problem~\cite{VFACKT23}. In this paper, we further this line of work by including heterogeneous values for demand.

%% file: acknowledgements.tex
\section*{Acknowledgments}
We sincerely thank the reviewers for their insightful feedback, and the shepherd for their support with the revision process. 